\documentclass[a4paper,twocolumn,10pt]{article}

\usepackage{graphicx}          
\usepackage{amsmath}           
\usepackage{amssymb}
\usepackage{algorithm}
\usepackage{algpseudocode}
\usepackage{hyperref}
\usepackage{color}
\usepackage{cite}
\usepackage{enumerate}
\usepackage{epstopdf}
\usepackage{subcaption}
\usepackage{appendix}

\usepackage[latin1]{inputenc} 
\usepackage{fancyhdr}
\usepackage{amsthm}

\theoremstyle{definition}

\theoremstyle{remark}
\theoremstyle{plain}
\newtheorem{fact}{Result}
\theoremstyle{remark} \newtheorem{note}{Remark}

\DeclareMathOperator{\E}{E}
\DeclareMathOperator{\Cov}{Cov}
\DeclareMathOperator{\Var}{Var}
\DeclareMathOperator{\tr}{tr}
\DeclareMathOperator*{\argmin}{arg\,min}

\DeclareMathOperator{\col}{col}
\DeclareMathOperator{\diag}{diag}
\newcommand{\mbf}[1]{\mathbf{#1}}
\newcommand{\mbs}[1]{\boldsymbol{#1}}
\newcommand{\what}[1]{\widehat{#1}}
\newcommand{\wtilde}[1]{\widetilde{#1}}
\newcommand{\wbar}[1]{\overline{#1}}

\newcommand{\dataset}{\mathcal{D}}
\newcommand{\risk}{\mathcal{R}}
\newcommand{\eps}{\varepsilon}
\newcommand{\Ehat}{\what{\text{E}}}
\newcommand{\erisk}{R}

\newcommand{\epsvec}{\mbs{\varepsilon}}
\newcommand{\X}{\mathbf{X}}
\newcommand{\U}{\mathbf{U}}
\newcommand{\Lam}{\mbs{\Theta}}
\newcommand{\regcM}{\mbs{\Psi}}
\newcommand{\regM}{\mbs{\Phi}}
\newcommand{\ycov}{\mbs{\Sigma}}
\newcommand{\ycross}{\mathbf{r}}
\newcommand{\scov}{\mathbf{Z}}
\newcommand{\y}{\mathbf{y}}
\newcommand{\covbasis}{\psi}
\newcommand{\regc}{\mbs{\psi}}
\newcommand{\reg}{\mbs{\phi}}
\newcommand{\regcol}{\widetilde{\mbs{\phi}}}
\newcommand{\ux}{\mathbf{u}}
\newcommand{\x}{\mathbf{x}}
\newcommand{\z}{\mathbf{z}}
\newcommand{\nusn}{\coeff_0}

\newcommand{\lin}{\mbs{\lambda}}
\newcommand{\lag}{\mbs{\kappa}}
\newcommand{\proj}{\mbs{\Pi}}

\newcommand{\covp}{\mbs{\theta}}
\newcommand{\covpsc}{\theta}
\newcommand{\T}{\top}

\newcommand{\coeff}{\mbf{w}}

\newcommand{\weight}{\mbs{\varphi}}
\newcommand{\weightsc}{\varphi}
\newcommand{\weightM}{\mbf{D}}

\newcommand{\noteend}{$\Diamond$}

\newcommand{\I}{\mbf{I}}
\newcommand{\0}{\mbf{0}}

\newcommand{\LR}{\textsc{Lr}}
\newcommand{\LC}{\textsc{Lc}}

\begin{document}


\title{\textbf{Online Learning for Distribution-Free Prediction}}
\author{Dave Zachariah, Petre Stoica and Thomas B. Schön\thanks{This
    work has been partly supported by the Swedish Research Council
    (VR) under contracts 621-2014-5874, 621-2013-5524 and 2016-06079.}}
\date{} 
\maketitle

\begin{abstract}
We develop an online learning method for prediction, which is
important in problems with large and/or streaming data
sets. We formulate the learning approach using a covariance-fitting
methodology, and show that the resulting predictor has desirable
computational and distribution-free properties: It is implemented
online with a runtime that scales
linearly in the number of samples; has a constant memory requirement; avoids
local minima problems; and prunes away redundant feature dimensions without relying on restrictive
assumptions on the data distribution. In conjunction with the split
conformal approach, it also produces distribution-free prediction confidence
intervals in a computationally efficient manner. The method is demonstrated on both real and synthetic datasets.
\end{abstract}

\section{Introduction}

Prediction is a classical problem in statistics, signal processing,
system identification and machine learning \cite{HastieEtAl2009_elements,Cressie&Wikle2011_spatiotemporal,Soderstrom&Stoica1988_system,Murphy2012_machine}. The
prediction of stochastic processes was pioneered in the temporal
and spatial domains, cf.
\cite{Wold1938_study,Kolmogorov1941_interpolation,Wiener1949_extrapolation,Cressie1990_origins},
but the fundamental ideas were generalized to arbitrary domains. In general,
the problem can be formulated as predicting the output
of a process, $y$, for a given input test point $\x$ after observing a
dataset of input-output pairs
$$\dataset = \bigl \{ (\x_1, y_1), \: \dots, \: (\x_n, y_n)  \bigr\}.$$

In this paper, we are interested in learning predictor functions
$\what{y}(\x)$ in scenarios where $n$ is very large or is
increasing. We consider two common classes of predictors:
\begin{enumerate}[1)]
\item The first class consists of predictors in \emph{linear
    regression} (\LR) form. That is, given
a regressor function $\reg(\x)$, the predictor $\what{y}(\x)$ is expressed as a linear combination of its elements. This class includes ridge
regression, \textsc{Lasso}, and elastic net
\cite{Hoerl&Kennard1970_ridge,Tibshirani1996_regression,Zou&Hastie2005_regularization}. The
regressor $\reg(\x)$ can be understood as a set of input features and
a standard choice for it is simply $\x$.

\item The second class comprises predictors in \emph{linear combiner} (\LC) form. That is, the
predictor $\what{y}(\x)$ is expressed as a linear combination of all
observed samples $\{ y_i \}$, using a model of the process. This class includes kernel smoothing,
Gaussian process regression, and local polynomial regression
\cite{Nadaraya1964_estimating,Watson1964smooth,DruckerEtAl_1997support,Williams&Rasmussen1996_gaussian,Cleveland1981_lowess}. A
standard model choice is a squared-exponential kernel or covariance
function for $y$.
\end{enumerate}
The weights in the \LR{} and \LC{} classes are denoted $\coeff$ and
$\lin$, respectively. For a  given prediction method in either class,
the weights are functions of the data $\dataset$.  The set
of possible weights is large and must be constrained in a balanced
manner to avoid a high variance and bias of $\what{y}(\x)$, which
correspond to overfitting and underfitting, respectively.

In most prediction methods, the constraints on the set of weights are
controlled using a hyperparameter, which we denote $\covp$ and which is
learned from data. There exist several methods for doing so,
including cross-validation, maximum likelihood, and weighted least-squares
\cite{HastieEtAl2009_elements,MacKay2003_information,Cressie1985_fitting}. These
learning methods are, however, nonconvex and not readily
scalable to large $n$, which is the target case of this paper. Moreover,
in the process of quantifying the prediction uncertainty, using
e.g. the bootstrap \cite{Efron&Tibshirani1986_bootstrap} or conformal approaches
\cite{VovkEtAl2005_algorithmic,LeiEtAl2016_distribution}, the learning methods compound the
complexity and render such uncertainty quantification intractable
for large $n$.

In this paper, we consider a class of predictors $\what{y}(\x ;
\covp)$ that can equivalently be expressed in 
either linear regression or combiner form. In this class, $\covp$ constrains
the weights for each dimension of $\reg(\x)$
individually. Thus irrelevant features can be suppressed and
overfitting mitigated
\cite{Neal1996_bayesian,Tipping2001_sblrvm,Murphy2012_machine}. To learn the hyperparameters
$\covp$, we employ a covariance-fitting
methodology \cite{Cressie1985_fitting,Anderson1989_linear,OtterstenEtAl1998_covariance} by generalizing a fitting criterion used in
\cite{StoicaEtAl2011_newspectral,StoicaEtAl2011_spice,Zachariah&Stoica2015_onlinespice}. We
extend this learning approach to a predictive
setting, which results in a predictor with the following attributes:
\begin{itemize}
\item computable online in linear runtime,
\item implementable with constant memory,
\item does not suffer from local minima problems,
\item enables tractable distribution-free confidence intervals,
\item prunes away irrelevant feature dimensions.
\end{itemize}
These facts render the predictor particularly suitable for scenarios with
large and/or growing number of data points. It can be viewed as an
online, distribution-free alternative to the automatic relevance
determination approach
\cite{Neal1996_bayesian,Tipping2001_sblrvm,Murphy2012_machine}. Our contributions include
generalizing the covariance-fitting methodology to non-zero mean
structures, providing connections between linear regression and combiner-type
predictors, and an analysis of prediction performance when the
distributional form of the data is unknown.

The remainder of the paper is organized as follows. In
Section~\ref{sec:background}, we introduce the problem of learning
hyperparameters. In Section~\ref{sec:distributionfree}, we highlight desirable
constraints on the weights in an \LR{} setting, whereas
Section~\ref{sec:modelbased} highlights these constraints in an \LC{}
setting. The hyperparameters constrain the weights in different ways
in each setting. In Section~\ref{sec:covariance}, 
the covariance-fitting based learning method is introduced and applied
to an \LC{} predictor. The resulting computational and
distribution-free properties are derived. Finally, in
Section~\ref{sec:experiments}, the proposed online learning approach
is compared with the offline cross-validation approach on a series of
real and synthetic datasets.

\begin{note}
In the interest of reproducible research, we have made the code for
the proposed method available at
\url{https://github.com/dzachariah/online-learning}. \noteend
\end{note}

\emph{Notation:} $\odot$ is the elementwise Hadamard product. The operation $\col\{ \mbf{x}_1,
\dots, \mbf{x}_n \}$ stacks all $\x_i$ into a single column vector, while
$[\mbf{X}]_i$ denotes the $i$th column of $\mbf{X}$. The sample mean
is written as $\Ehat[\x_i] = n^{-1}\sum^{n}_{i=1} \x_i$. The number of
nonzero elements in a vector is denoted as $\| \x \|_0$. The Kronecker
delta is denoted $\delta(\x,\y)$. 

\emph{Abbreviations:} Independent and identically distributed (i.i.d.).
 
\section{Learning problem}
\label{sec:background}

Let $\what{y}(\x ; \covp)$ denote a predictor with hyperparameters
$\covp$. For linear regression and combiner-type predictors, the choice of
$\covp$ constrains the set of weights. 
For any input-output pair $(\x, y)$, the risk of the predictor
is taken to be the mean-squared error \cite{Wasserman2004_allofstats}
\begin{equation}
\risk \triangleq \E \left[ \left| y - \what{y}(\x ; \covp) \right|^2
\right].
\label{eq:risk}
\end{equation}

The optimal choice of $\covp$ is therefore the hyperparameter that
minimizes the unknown risk. A common distribution-free learning
approach  is
cross-validation, which
estimates the risk for a fixed choice of $\covp$. The risk estimate is formed by first dividing the training data into $K$ subsets and then predicting the output in one
subset using data from the remaining $K-1$ subsets\cite[ch.~7]{HastieEtAl2009_elements}:
\begin{equation}
\what{\risk}(\covp)\: = \: \sum^K_{k=1}\frac{n_k}{n} \: \Ehat\left[ \left| y_i -
  \what{y}_{\neg k}(\x_i; \covp) \right|^2 \right],
\label{eq:learning_kfoldcv}
\end{equation}
where $n_k$ is the number of samples in subset $k$ and $\what{y}_{\neg k}(\x; \covp)$ denotes the predictor using
training data from all subsets except $k$. In this approach, the hyperparameter is learned by finding $\what{\covp}$ that
minimizes $\what{\risk}(\covp)$. For one-dimensional
hyperparameters, a numerical search for the minimum of
\eqref{eq:learning_kfoldcv} is feasible when $K$ is
small; otherwise it may well be impractical.

An alternative learning approach is to assume that the distribution of $(\x, y)$
belongs to a family that is parameterized by $\covp$. Then
the hyperparameter can be learned by fitting a covariance
model of the data to the empirical moments of the training data, 
cf. \cite{Cressie1985_fitting}. If additional assumptions are made
about the distributional form, it is possible to formulate a complete
probabilistic model of the data and learn $\covp$ using the asymptotically efficient maximum likelihood approach,
cf. \cite{Stein1999_interpolation,MacKay2003_information}. 

The aforementioned statistical learning methods are, however, in
general nonconvex and may therefore give rise to multiple minima problems
\cite{Mardia&Watkins1989_multimodality,Zimmermann2011_asymptotic}. This
becomes problematic when $\covp$ is multidimensional as the methods may require
a careful choice of initialization and numerical search techniques
\cite[ch.~5]{Rasmussen&Williams2006_gaussian}. In addition, the distributional assumptions employed in the maximum likelihood approach may lack robustness to model misspecifications \cite{Carroll&Ruppert1982_comparison}.

More importantly for the data scenarios considered in this paper,
these learning methods cannot readily be implemented online. That is, for a
given dataset, the computational complexity of learning $\covp$
does not scale well with $n$ and the process must be repeated each time $n$
increases. Consequently, for each new
$\what{\covp}$,  the predictor function $\what{y}(\x ; \what{\covp})$ must be computed
afresh.

Our main goal is to develop a learning approach that
is implementable online, obviates local minima problems, and has desirable
distribution-free properties. Before we introduce the proposed
learning approach, we introduce the hyperparameters in the context of
\LR{} and \LC{} predictors, respectively.

\section{Linear regression predictor}
\label{sec:distributionfree}

\LR{} predictors are written in the following form
\begin{equation}
\boxed{\what{y}(\x) = \reg^\T(\x) \coeff,}
\label{eq:generallinpredictor}
\end{equation}
where $\reg(\x)$ is a $p$-dimensional regression function and $\coeff$
are weights. The empirical risk of a predictor can then be written as 
\begin{equation}
\erisk(\coeff) \triangleq \Ehat\left[|y_i - \reg^\T(\x_i) \coeff|^2 \right].
\label{eq:risk_empirical}
\end{equation}
When the empirical risk is minimized under a constraint on $\coeff$,
irrelevant dimensions of $\reg(\x)$ can be suppressed which mitigates overfitting 
\cite{HastieEtAl2009_elements,LazerEtAl2014_Google_flu}. For
notational convenience we write the regressor matrix as
\begin{equation*}
\regM \triangleq \begin{bmatrix}
\reg^\T(\x_1) \\ \vdots \\ \reg^\T(\x_n)
\end{bmatrix} = \begin{bmatrix}
\regcol_1 & \cdots & \regcol_p
\end{bmatrix},
\end{equation*}
where $\regcol_j$ denotes the $j$th column.

When $p$ is large, we would ideally like to find a subset of at most $k \ll p $
relevant dimensions
\cite{HastieEtAl2009_elements,WestonEtAl2003_zeronorm,Foster&George1994_risk,HastieEtAl2015_learningsparse}. We
may then define the optimal weights as
\begin{equation}
\begin{split}
\coeff_\star \: &\triangleq \: \argmin_{\coeff \: : \: \| \coeff  \|_0
  \leq k} \; \erisk(\coeff) .
\label{eq:coeffmin_sparse}
\end{split}
\end{equation}
Let $\what{y}_\star(\x) = \reg^\T(\x)
\coeff_\star$ denote corresponding optimal \LR{} 
predictor. This sparse predictor does not rely on any
assumptions on the data, and is therefore robust to
model misspecifications. We collect the prediction errors of the
training data in the vector
$$\epsvec = \col\{ \eps_1, \dots, \eps_n \},$$
where  $\eps_i = y_i - \what{y}_\star(\x_i)$. In the subsequent analysis, we
also make use of 
\begin{equation}
\eps_\star \; \triangleq \; \max_{1 \leq j \leq p} \; \left| \epsvec^\T \regcol_j
\right |.
\label{eq:eps_constant}
\end{equation}

Solving \eqref{eq:coeffmin_sparse}
is a computationally intractable problem,
even for moderately sized $p$. Thus $\what{y}_\star(\x)$ will be taken as a
reference against which any alternative $\what{y}(\x)$ can be compared
by defining the prediction divergence \cite{Buhlmann&VanDeGeer2011_highdim}:
\begin{equation}
\begin{split}
\Delta_\star \: \triangleq \: \Ehat\left[  \bigl| \what{y}(\x_i) -
  \what{y}_\star(\x_i)  \bigr|^2 \right].
\end{split}
\end{equation}

A tractable convex relaxation of the problem in
\eqref{eq:coeffmin_sparse} is the $\ell_1$-regularized
\textsc{Lasso} method
\cite{Tibshirani1996_regression}, as defined by
\begin{equation}
\begin{split}
\what{\coeff} \: &= \: \argmin_{\coeff} \;
\erisk(\coeff) + \covpsc\| \coeff \|_1,
\label{eq:coeffmin_lasso}
\end{split}
\end{equation}
where $\covpsc \geq 0$ is a hyperparameter that shrinks all weights
toward zero in a sparsifying fashion. This \LR{} predictor
has desirable risk-minimizing properties and can be computed online
in linear runtime for any fixed $\covpsc$
\cite{Greenshtein&Ritov2004_persistence,AngelosanteEtAl2010_online}. The
corresponding divergence can be bounded per the following result.
\begin{fact}
If the hyperparameter satisfies
\begin{equation}
\covpsc \geq  \frac{2\eps_\star}{n},
\label{eq:assumption_lasso}
\end{equation}
then the divergence of the \textsc{Lasso}-based predictor $\what{y}(\x)$
from the optimal sparse predictor $\what{y}_\star(\x)$ is bounded by
\begin{equation}
\Delta_\star \: \leq \: 2\covpsc\| \coeff_\star \|_1.
\label{eq:bound_lasso}
\end{equation}
\end{fact}

\begin{proof}
See Appendix~\ref{app:bound_lasso}
\end{proof}

If the hyperparameter $\covpsc$ in \eqref{eq:coeffmin_lasso} satisfies \eqref{eq:assumption_lasso}, then
\eqref{eq:bound_lasso} ensures that the maximum divergence of
\textsc{Lasso}-based predictor from the optimal \LR{} predictor
remains bounded. Then redundant dimensions
of $\reg(\x)$ are  pruned away, which is a desirable property. The hyperparameter $\covpsc$ in
\eqref{eq:coeffmin_lasso} is typically learned offline using
\eqref{eq:learning_kfoldcv}, which is a limitation for the scenarios
of large and/or increasing $n$ considered herein. Moreover, the individual weights in \eqref{eq:generallinpredictor} are
constrained uniformly in the \textsc{Lasso} approach. To provide a more
flexible predictor that constrains the weights individually, we
consider an \LC{} approach next.

\begin{note}
For the special case in which the prediction errors of $\what{y}_\star(\x)$
are i.i.d. zero mean Gaussian with variance $\sigma^2$,
and the regressors are normalized such that $\| \regcol_j \|^2 \equiv n$, the inequality \eqref{eq:assumption_lasso} is satisfied with high
probability when $\covpsc = 
\sigma \sqrt{ (2  \ln p + \delta) / n}$ for some positive constant $\delta$, cf.
\cite{Buhlmann&VanDeGeer2011_highdim}. However, this is an infeasible choice since $\sigma^2$ is unknown and must be estimated. \noteend
\end{note}

\section{Linear combiner predictor}
\label{sec:modelbased}

\LC{} predictors are written in the following form
\begin{equation}
\boxed{\what{y}(\x) = \lin^\T(\x) \y,}
\label{eq:generallinsmoother}
\end{equation}
where $\y = \col\{ y_1, \dots, y_n \}$ contains the $n$ observed samples
and $\lin$ are weights. Given any $\x$, we will find an unbiased
predictor of $y$ that minimizes the conditional risk. That is, we seek the solution to
\begin{equation}\label{eq:BLUP_problem}
\begin{split}
\min_{\lin} \;  \risk(\lin|\x) , \quad \text{subject to} \:
\E\left[  y- \lin^\T \y  \: | \: \X, \x \right] = 0,
\end{split}
\end{equation} 
where
$$\risk(\lin|\x) = \E\left[ \bigl| y - \lin^\T \y \bigr|^2 \: | \:
  \X, \x \right],$$
and $\X$ denotes all inputs in $\dataset$. The weights $\lin$ are then
constrained by assuming a model of $y$ given $\x$.

Here we specify a simple model that does not rely on the
distributional form of $y$ but merely constrains its moments \cite{Stein1999_interpolation}:
\begin{itemize}
\item The conditional mean of $y$ is parameterized as 
\begin{equation}
\label{eq:function_mean}
\E[y|\x] = \ux^{\T}(\x) \nusn,
\end{equation}
where $\ux(\x)$ is a given $u \times 1$ function and $\nusn$ are unknown
coefficents. In the simplest case we have an unknown constant mean,
i.e., $\E[y|\x] \equiv w_0$, by setting $\ux(\x) \equiv
1$. Alternatively, one may consider an unknown affine function. 

\item The conditional covariance function can be
expressed as
\begin{equation}
\begin{split}
\Cov[y, y'|\x,\x'] &= \sum^{\infty}_{k=1} \covpsc_k \covbasis_k(\x)
\covbasis_k(\x') + \covpsc_0 \delta(\x, \x')
\end{split}
\label{eq:function_cov}
\end{equation}
using
Mercer's theorem
\cite{Stein1999_interpolation},  where $\covbasis_k(\x)$ is given and $\theta_k$ are nonnegative 
parameters. For a stationary process with an isotropic covariance
function, we may for instance use the Fourier or Laplace operator
basis $\covbasis_k(\x)$ \cite{Rahimi&Recht2007_random,Solin&Sarkka2014_hilbert}. 
\end{itemize}

To enable an computationally efficient online implementation, we consider a model with a truncated sum
of $q$ terms in \eqref{eq:function_cov}. We write
$$\regc(\x) = \col\{  \covbasis_1(\x), \dots, \covbasis_q(\x) \}$$
and the hyperparameter as
$$\covp \triangleq \col\{ \covpsc_0, \covpsc_1, \dots, \covpsc_q  \},$$ 
for notational simplicity.

\begin{note}
The model, specified by \eqref{eq:function_mean} and
\eqref{eq:function_cov}, is invariant with respect to the
distributional properties of the data. It therefore includes commonly
assumed distributions, such as Gaussian and Student-t
\cite{Rasmussen&Williams2006_gaussian,ShahEtAl2014_student}. \noteend
\end{note}

\begin{note}
When $y$ is modelled as a stationary process, 
the Fourier or Laplace
operator basis can be used in \eqref{eq:function_cov} to parameterize its
power spectral density via $\{ \covpsc_k \}$. This can be interpreted as
a way to parameterize the smoothness of the process \cite{Stein1999_interpolation}. \noteend
\end{note}

\begin{fact}
 Under
 model assumptions \eqref{eq:function_mean} and
 \eqref{eq:function_cov}, the  covariance properties of the
 training data are given by
\begin{equation}\label{eq:statisticalmoments_cov}
\begin{split}
\ycov &\triangleq \Cov[\y| \X] = \regcM \Lam \regcM^\top + \covpsc_0
\mbf{I}_n \\
\ycross(\x) &= \Cov[\y, y | \X, \x] = \regcM \Lam \regc(\x) ,
\end{split}
\end{equation}
where \begin{equation*}
\begin{split}
\regcM &= \begin{bmatrix} \regc(\mbf{x}_1) &
  \cdots & \regc(\mbf{x}_n) \end{bmatrix}^\T,\\
\Lam &= \diag(\covpsc_1, \dots, \covpsc_q).
\end{split} 
\end{equation*}
\end{fact}

\begin{proof}
The result follows from an elementwise application of \eqref{eq:function_cov}
to the training data:
\begin{equation*}
\begin{split}
\Cov[y_i,y_j | \x_i, \x_j]&=\sum^q_{k=1} \covpsc_k \covbasis_k(\x_i)
\covbasis_k(\x_j) + \covpsc_0 \delta(i,j),\\
\Cov[y_i,y| \x_i, \x] &= \sum^q_{k=1} \covpsc_k \covbasis_k(\x_i)
\covbasis_k(\x) .\\
\end{split}
\end{equation*}
\end{proof}

\begin{fact}
 The weights of the optimal \LC{} predictor $\what{y}(\x; \covp)$ in \eqref{eq:generallinsmoother} are given by
\begin{equation}
\begin{split}
\lin(\x) &= \ycov^{-1} \U (\U^\top \ycov^{-1} \U)^\dagger \ux(\x) + \ycov^{-1}
\proj^{\perp} \ycross(\x),
\end{split}
\label{eq:optimallin}
\end{equation}
where
\begin{equation*}
\U = \begin{bmatrix} \mbf{u}(\mbf{x}_1) & 
  \cdots & \mbf{u} (\mbf{x}_n) \end{bmatrix}^\T .
\end{equation*}
and $\proj^{\perp}$ is a projector onto $\text{span}(\U)^{\perp}$.
\end{fact}

\begin{proof}
See Appendix~\ref{app:optimalfilter}.
\end{proof}

Consequently, the hyperparameters in $\covp$ constrain the optimal weights
$\lin$ via the model \eqref{eq:function_mean} and
 \eqref{eq:function_cov}. If the model were correctly specified, we
 could in principle search for a risk-minimizing $\covp$. This would,
 however, be intractable in an online setting.

Next, we show that it is possible to interchange \LC{} and \LR{}  
representations, \eqref{eq:generallinsmoother} and
\eqref{eq:generallinpredictor}, of the predictor
$\what{y}(\x; \covp)$.
\begin{fact}
The optimal \LC{} predictor in \eqref{eq:generallinsmoother} and \eqref{eq:optimallin} can be written in \LR{} form
\begin{equation}
\what{y}(\x ; \covp) = \reg^\T(\x) \what{\coeff},
\label{eq:predictor_modelbased}
\end{equation}
where 
\begin{equation}
\reg(\x) = \col\{   \ux(\x) , \regc(\x) \}
\label{eq:regressor_smooth}
\end{equation}
and the $p=u+q$ weights are given by
\begin{equation}
\what{\coeff} = \argmin_{\coeff} \: \erisk(\coeff) +
\frac{\covpsc_0}{n}\| \coeff \|^2_{\weightM},
\label{eq:coeffmin_smooth}
\end{equation}
where $\weightM = \diag(\0, \Lam^{-1})$.
\end{fact}

\begin{proof}
See Appendix~\ref{app:equivalence}.
\end{proof}

\begin{note}
Note
that, if we impose $\covpsc_1 = \cdots = \covpsc_q \equiv
1$, \eqref{eq:coeffmin_smooth} is equivalent to the ridge regression
method \cite{Hoerl&Kennard1970_ridge}. \noteend
\end{note}

\begin{note}
The \LC{} predictor has an \LR{} intepretation: when \eqref{eq:function_cov} is formed using the Fourier or
Laplace operator basis, the corresponding regressor \eqref{eq:regressor_smooth}
will approximate any
isotropic covariance function by varying
$\covp$. This
includes covariance functions belonging to the important Mat\'{e}rn
class, of which the popular
squared-exponential function is a special case
\cite[ch.~6.5]{Stein1999_interpolation}. Cf. \cite{Rahimi&Recht2007_random,Solin&Sarkka2014_hilbert,FelixEtAl2016_orthogonal}
for regressor functions $\reg(\x)$ with good approximation
properties. Similarly, by partitioning an arbitrary regressor $\reg(\x)$ into the form \eqref{eq:regressor_smooth} yields an
\LC{} interpretation of an \LR{} predictor. The regressor matrix is then
$\regM =[ \U \: \regcM]$.  For instance, a standard
regressor is the affine function $\reg(\x) = \col\{ 1, \x \}$, which
can be interpreted as a constant mean with a covariance function that
is quadratic in $\x$. \noteend
\end{note}



The elements of $\covp$ constrain each of the weights \eqref{eq:coeffmin_smooth} individually, unlike the uniform approach in
\eqref{eq:coeffmin_lasso}. Thus the hyperparameter in the optimal
\LC{} predictor
determines the relevance of individual features in $\reg(\x)$, similar
to the automatic relevance determination framework in
\cite{Tipping2001_sblrvm,Faul&Tipping2002_analysis,Wipf&Nagarajan2008_ard}. While
this enables a more flexible predictor,  the $\ell_2$-regularization
term in \eqref{eq:coeffmin_smooth} does not yield the sparsity
property of the \textsc{Lasso}-based predictor. More importantly,
the learning approach \eqref{eq:learning_kfoldcv} is intractable when
$\covp$ is multidimensional. We show how to get around this issue in the next section.


\section{Online learning via covariance fitting}
\label{sec:covariance}

We consider a covariance-fitting approach for learning $\covp$ in
the flexible \LC{} predictor. We show that using the learned
hyperparameter results in a predictor with desirable, computational and
distribution-free, properties.

The normalized sample covariance matrix of the training data can be written as
\begin{equation*}
\begin{split}
\wtilde{\scov} &= \frac{ (\y - \U \nusn) (\y - \U \nusn)^\top}{\| \y -
  \U \nusn \|_2},
\end{split}
\end{equation*}
using the mean structure \eqref{eq:function_mean}. The assumed
covariance structure $\ycov$ is
parameterized by $\covp$ in \eqref{eq:statisticalmoments_cov}. We seek to fit $\ycov$ to
the sample covariance in the following sense:
\begin{equation}
\begin{split}
\boxed{\covp^\star = \argmin_{\covp}  \quad \bigl \| \: \widetilde{\scov} -
\ycov \: \bigr \|^2_{\ycov^{-1}},}
\end{split}
\label{eq:covariancematch_original}
\end{equation}
using a weighted norm that penalizes correlated residuals, cf. \cite{Cressie1985_fitting,Anderson1989_linear,OtterstenEtAl1998_covariance}. To match
the magnitude of the normalized sample covariance we 
subject the parameters in \eqref{eq:covariancematch_original} to the normalization constraint
\begin{equation}
\text{tr}\bigl\{ \wtilde{\scov} - \ycov \bigr \} = 0.
\label{eq:covariancematch_original_constraint}
\end{equation}
\begin{fact}
The learning problem defined by \eqref{eq:covariancematch_original} is convex in $\covp$ and will therefore not suffer from local minima issues.
\end{fact}

\begin{proof}
See Appendix~\ref{app:convexity}.
\end{proof}

\begin{note}
Eq. \eqref{eq:covariancematch_original} is a generalization of the
 covariance-fitting criterion in
\cite{StoicaEtAl2011_newspectral,StoicaEtAl2011_spice} to the case of
nonzero mean structures of the data. See also
\cite{StoicaEtAl2014_weightedspice} for further connections. \noteend
\end{note}


\subsection{Properties of the resulting predictor}

Using \eqref{eq:covariancematch_original} in the \LC{} predictor from
\eqref{eq:generallinsmoother} and \eqref{eq:optimallin}, we obtain the
following result:
\begin{fact}
The \LC{} predictor with a learned
parameter $\covp^\star$ has the following \LR{} form:
\begin{equation}
\what{y}(\x; \covp^\star) = \reg^\T(\x) \what{\coeff}
\label{eq:spice_predictor}
\end{equation}
where
\begin{equation}
\boxed{\what{\coeff} \: = \: \argmin_{\coeff} \; \sqrt{\erisk(\coeff)}
\: +\: \frac{1}{\sqrt{n}} \| \weight \odot \coeff  \|_1}
\label{eq:coeffmin_spice}
\end{equation}
and the elements of $\weight$ are set to
\begin{equation}
\weightsc_j = \begin{cases}
\frac{1}{\sqrt{n}} \| \regcol_j \|_2, & j > u \\
0, & \text{otherwise}
\end{cases}.
\label{eq:weight_spice}
\end{equation}
\end{fact}

\begin{proof}
See Appendix~\ref{app:spice_predictor}.
\end{proof}
Thus the covariance-based learning approach endows the predictor with a sparsity
property, via the weighted $\ell_1$-regularization term in
\eqref{eq:coeffmin_spice}, that prunes away redundant feature dimensions. Similar to
\eqref{eq:bound_lasso}, we can bound the associated
prediction divergence.

\begin{note}
The methodology in
\eqref{eq:covariancematch_original} generalizes the approach in
\cite{StoicaEtAl2011_newspectral,StoicaEtAl2011_spice}, and following the
cited references we call
\eqref{eq:spice_predictor} the \textsc{Spice} (\emph{sp}arse \emph{i}terative
\emph{c}ovariance-based \emph{e}stimation)
predictor. Eq. \eqref{eq:coeffmin_spice} can be viewed as a
nonuniformly weighted extension of the square-root \textsc{Lasso}
method in \cite{BelloniEtAl2011_sqrtlasso}. \noteend
\end{note}

\begin{fact}
If all elements \eqref{eq:weight_spice} are positive $(u=0)$ and satisfy
\begin{equation}
\weightsc_j \; \geq \;
\frac{\eps_\star}{\sqrt{n\erisk(\coeff_\star)}}, 
\label{eq:assumption_spice}
\end{equation}
then the divergence of the \textsc{Spice} predictor
$\what{y}(\x)$ from the optimal sparse predictor $\what{y}_\star(\x)$ is bounded by
\begin{equation}
\Delta_\star \; \leq \; \frac{2}{n}\| \weight \odot
\coeff_\star \|^2_1  + 4 \sqrt{\frac{\erisk(\coeff_\star)}{n}} \| \weight \odot
\coeff_\star \|_1.
\label{eq:bound_spice}
\end{equation}
\end{fact}

\begin{proof}
See Appendix~\ref{app:bound_spice}.
\end{proof}

The above results are valid even when the model
\eqref{eq:function_mean} and \eqref{eq:function_cov} is misspecified.
We have therefore shown a distribution-free property of the
\textsc{Spice} predictor, which will prune away redundant feature dimensions automatically. It parallels \eqref{eq:bound_lasso} but is not
dependent on any hyperparameters. Thus we can view the
covariance-fitting approach as a vehicle for constructing a general \LR{}
predictor with weights on the form \eqref{eq:coeffmin_spice}.

In addition, the \textsc{Spice} predictor has computational properties
that are appealing in cases with large and/or increasing $n$, where an online implemention is desirable.
\begin{fact} The
  \textsc{Spice} predictor \eqref{eq:spice_predictor} can be updated at each new
  data point $(\x_i, y_i)$ in an online manner. The computations are based on
\begin{equation}
\begin{split}
\mbs{\Gamma} \triangleq \regM^\top \regM , \quad \mbs{\rho} \triangleq
\regM^\top \mbf{y} \quad \text{and} \quad \kappa  \triangleq \mbf{y}^\top \mbf{y},
\end{split}
\label{eq:variables_spice}
\end{equation}
which have fixed dimensions and are readily updated sequentially. The
pseudocode for the online learning algorithm is summarized in
Algorithm~\ref{alg:onlinesparse}. It is initialized by setting the
weights to
$\check{\coeff} = \0$. 
\end{fact}

\begin{proof}
See Appendix~\ref{app:spice_alg}.
\end{proof}

\begin{fact} The total runtime of the algorithm is
  $\mathcal{O}(nLp^2)$, where $L$ is the number of cycles per sample, and its memory requirement is
  $\mathcal{O}(p^2)$.
\end{fact}
\begin{proof}
The complexity of the loop in Algorithm~\ref{alg:onlinesparse} is proportional to $Lp^2$ for each sample. 
Regarding the memory, the number of stored variables is constant and the
largest variable is the $p \times p$ matrix $\mbs{\Gamma}$.
\end{proof}

\begin{note} In the given algorithm, the individual weights $\{ w_j
  \}$ are  updated in a fixed cyclic order. Other possible updating orders are
  described in a general context, in
\cite{HongEtAl2016_unified}. As $L \rightarrow \infty$, however, the
algorithm converges to the global minimizer
\eqref{eq:coeffmin_spice} at each sample $n$,
irrespective of the order in which $\{ w_j \}$ are updated
\cite{Zangwill1969_nonlinear,Stoica&Selen2004_cyclic}. \noteend
\end{note}

\begin{note} Let $V_n(\cdot)$ denote the convex cost function in
\eqref{eq:coeffmin_spice} at sample $n$. The termination point of the
cyclic algorithm for $V_{n-1}(\cdot)$ provides the starting point for
minimizing the subsequent cost function $V_{n}(\cdot)$. For 
finite $L$, the termination point will deviate
from the global minimizer \eqref{eq:coeffmin_spice}. In practice,
however, we found that $L$ can be chosen as a small integer and that even
$L=1$ yields good prediction results when $n > p$. Furthermore, we
observed that the results are robust with respect to the orders in
which $\{ w_j \}$ are updated. \noteend
\end{note}

\begin{note} In the special case when the prediction errors of
  $\what{y}_\star(\x)$ are i.i.d. zero-mean Gaussian and the
  regressors are normalized as $\| \regcol_j \|^2_2 \equiv n$, the bound
  \eqref{eq:bound_spice} can be ensured with high probability by
  multiplying the coefficients \eqref{eq:weight_spice} by a factor $2 \sqrt{2 \ln
    p + \delta}$, where $\delta$ is a positive constant. See
  Appendix~\ref{app:bound_spice}. \noteend
\end{note}

\begin{algorithm}
  \caption{: Online learning via covariance fitting} \label{alg:onlinesparse}
\begin{algorithmic}[1]
    \State Input: $\x_n$, $y_n$ and $\check{\coeff}$
    \State $\mbs{\Gamma} := \mbs{\Gamma} + \reg(\x_n) \reg^\top(\x_n)$
    \State $\mbs{\rho} := \mbs{\rho} + \reg(\x_n) y_n$ 
    \State $\kappa := \kappa + y_n^2$
    \State $\xi = \kappa + \check{\coeff}^\top \mbs{\Gamma} \check{\coeff} - 2 
    \check{\coeff}^\top \mbs{\rho} $   
  \State $\mbs{\zeta} = \mbs{\rho} - \mbs{\Gamma} \check{\coeff}$
    \Repeat
        \State $j = 1, \dots, p$
        \State Compute $\hat{w}_j$ using \eqref{eq:zi_hat_p} $(j \leq u)$ otherwise \eqref{eq:zi_hat_pq}
        \State $\xi := \xi  + \Gamma_{jj} (\check{w}_j - \hat{w}_j)^2 + 2 (\check{w}_j - \hat{w}_j) \zeta_j $        
        \State  $\mbs{\zeta}  := \mbs{\zeta} + [\mbs{\Gamma}]_j (\check{w}_j - \hat{w}_j)$ 
        \State $\check{w}_j := \hat{w}_j$
    \Until{ number of iterations equals $L$ }
    \State Output: $\what{\coeff}$
\end{algorithmic}
\end{algorithm}

\subsection{Distribution-free prediction and inference}

In summary, the covariance-fitting methodology above has the following
main attributes:
\begin{itemize}
\item avoids local minima problems in the learning process,
\item results in a predictor $\what{y}(\x)$ that can be implemented online
\item maximum divergence from the optimal sparse predictor
  $\what{y}_\star(\x)$ can be evaluated.
\end{itemize}
Its computational and distribution-free properties also make it possible
to combine this approach with the split conformal method in
\cite{LeiEtAl2016_distribution}, which provides computationally
efficient uncertainty measures for a predictor $\what{y}(\x)$ under
minimal assumptions.

Suppose the input-output data consist of i.i.d. realizations
from an unknown distribution
$$(\x_i, y_i) \sim p(\x, y).$$
For a generic point $\x$, we would like to construct a
confidence interval for the predictor
\eqref{eq:spice_predictor} with a targeted coverage. That is,
find a finite interval
\begin{equation} 
C(\x) \triangleq \bigl [ \what{y}(\x) - \wbar{r} , \:  \what{y}(\x)
 + \wbar{r} \bigr],
\label{eq:interval_splitconformal}
\end{equation}
that covers the predicted output $y$ with a probability that reaches a
prespecified level $\kappa \in (0,1)$.

For simplicity, assume that $n$ is an even number and randomly
split $\dataset$ into two equally-sized datasets $\dataset'$ and
$\dataset''$. For a given targeted coverage level $\kappa$, the split conformal interval is constructed using the
following three steps \cite{LeiEtAl2016_distribution}:
\begin{enumerate}[1)]
\item Train the \textsc{Spice} predictor $\what{y}(\x)$ using $\dataset'$.
\item Predict the outputs in $\dataset''$ and compute the residuals $r_i = |y_i - \what{y}(\x_i)|$
\item Sort the residuals and let $\wbar{r}$ denote the
  $k$th smallest $r_i$, where $k= \lceil (n/2+1) \kappa \rceil$.
\end{enumerate}
\begin{fact}
Setting $\wbar{r}$ as above in \eqref{eq:interval_splitconformal},
yields an interval $C(\x)$ that 
covers the predicted output $y$ with a probability
$$ \Pr \bigl \{ y \in C(\x) \bigr \} \: \geq \: \kappa.$$
Thus the targeted level can be ensured. In addition, when the residuals $\{ r_i \}$ have a continuous distribution, the probability is also bounded from above by $\kappa + \frac{2}{n+2}$.
\end{fact}

\begin{proof}
See \cite[sec.~2.2]{LeiEtAl2016_distribution}.
\end{proof}

\begin{note}
Predictive probabilistic models, such as those considered in
\cite{Tipping2001_sblrvm,Wipf&Rao2007_ebayes}, provide credibility
intervals but, in contrast to the proposed approach, lack coverage guarantees. In fact, even when
the said models are correctly specified, their uncertainty is
systematically underestimated after learning the hyperparameters
\cite{WaagbergEtAl2016_prediction}. \noteend
\end{note}

The \textsc{Spice}-based split conformal prediction interval $C(\x)$ 
in \eqref{eq:interval_splitconformal} provides a
computationally efficient, distribution-free prediction and inference
methodology.


\section{Numerical experiments}
\label{sec:experiments}

In this section, we compare the online \textsc{Spice} approach developed
above with the well-established offline $K$-fold cross-validation
approach for learning predictors that use a given regressor function $\reg(\x)$. To
balance the bias and variance of \eqref{eq:learning_kfoldcv}, as well as the associated computational
burden, we
follow the recommended choice of $K=10$ folds
\cite[ch.~7]{HastieEtAl2009_elements}. Finding the global minimizer of
\eqref{eq:learning_kfoldcv} in high dimensions is intractable, and
in the following examples we restrict the discussion to the learning of predictors based on the
scalable Ridge and \textsc{Lasso} regression methods, since these require
only one hyperparameter.

Offline cross-validation was performed by
evaluating \eqref{eq:learning_kfoldcv} on a grid of 10 hyperparameter
values and selecting the
best value. Each evaluation requires retraining the predictor with a
new hyperparameter value and the entire process is computationally
intensive.

For the prediction problems below, we apply regression functions of
the form \eqref{eq:regressor_smooth}: 
$$\reg(\x) = \col\{ 1, \regc(\x) \},$$
using a constant mean $\ux(\x) \equiv 1$. For $\regc(\x)$ we use either
a linear function, $\regc(\x)  = \x$, or the Laplace operator basis
due to its attractive approximation properties
\cite{Solin&Sarkka2014_hilbert}. For the latter choice, suppose $\x$ is $d$-dimensional and belongs to $\mathcal{X} = [-L_1,L_1] \times \cdots \times  [-L_d,L_d]$. Then the elements of $\regc(\x)$ are defined by
\begin{equation}
\covbasis_{k_1, \dots, k_d}(\x) = \prod^{d}_{j=1} \frac{1}{\sqrt{L_j}}\sin\left( \frac{\pi k_j(x_j + L_j)}{2L_j} \right),
\label{eq:laplacebasis}
\end{equation}
where $k_j = 1, \dots, m$ are the indices for dimension $j$. We have that
$\regc(\x) = \col \{  \covbasis_{1, \dots, 1}(\x) \: \cdots  \: \covbasis_{m, \dots,
  m}(\x) \}$ has dimension $q=m^d$. The rectangular domain
$\mathcal{X}$ can easily be translated to any arbitrary point. When
$d$ is large, we may apply the basis to
each dimension $x_j$ separately. Then the resulting $\regc(\x) $ has
dimension $q = md$.

\subsection{Sparse linear regression}

In this experiment we study the learning of predictors under two challenging conditions: heavy-tailed noise and colinear regressors. The input $\x$ is of dimension $d = 100$. The dataset $\dataset$ was generated using an conditional Student-t distribution with mean $$\E[y| \x] = 1 + 5x_{1} + 5x_{10} + 5x_{20} + 5x_{30}
+ 5x_{40}$$ and variance $\Var[y | \x] = 4$. The input $\x$ were generated using an i.i.d. degenerate zero-mean Gaussian variable with covariance matrix $\mbf{C}_x$, where the numerical rank of $\mbf{C}_x \approx d/2 = 50$ and the variances are normalized by setting $\tr\{ \mbf{C}_x \} = d$.

For the predictors, we let $\regc(\x)$ be linear and thus $p=101$. We ran $10^3$ Monte Carlo simulations to
evaluate the performances of the predictors. In the first set of experiments
we estimate the risk $\risk$. To clarify the
comparison between the predictors, the risk is normalized by the noise
variance and presented in decibel scale (dB) in
Table~\ref{tab:exp_sparse_nmse}. As expected for this data generating
process, the sparse predictors outperform Ridge. For \textsc{Spice}
and \textsc{Lasso}, the difference is notable when $n<d$ but is less significant as more
samples are obtained.

Next, we evaluate the inferential properties by repeating the above
experiments with $n=2n'$ samples. The dataset $\dataset$ is randomly
partitioned into two sets $\dataset'$ and $\dataset''$, each of size
$n'$, to produce confidence intervals $C(\x)$ as in
\eqref{eq:interval_splitconformal}. We target the coverage level
 $\kappa = 0.90$, and report the average confidence interval length
as well as average coverage of the interval in
Table~\ref{tab:exp_sparse_coverage}. Note that the probability that $y
\in C(\x)$ is nearly exactly equal to the targeted level $\kappa$
without relying on any distributional assumptions. Thus the reported
confidence intervals are accurate. Furthermore, the average interval
lengths are significantly smaller for the sparse predictors compared
to Ridge. The reported intervals for \textsc{Spice} and \textsc{Lasso}
are similar in length, with the former being slightly smaller. The
interval lengths can be compared to the dynamic range of $y$, which
has a length of approximately 60.

\begin{table}
	\caption{Risk normalized by noise level [dB]}
	\begin{center}
	\begin{tabular}{c|c|c|c}
		$n$ & \textsc{Spice} & Ridge & \textsc{Lasso} \\
		\hline
		50  & $2.54$  & $10.28$  & $2.85$\\
     	100 & $1.07$  & $4.14$   & $1.15$\\
		200 & $0.32$  & $2.73$   & $0.41$
	\end{tabular}
	\end{center}
        \label{tab:exp_sparse_nmse}
\end{table}

\begin{table}
	\caption{Average confidence interval length with target coverage level $\kappa = 0.90$. Average coverage level of interval in parenthesis.}
	\begin{center}
	\begin{tabular}{c|c|c|c}
		$n'$ & \textsc{Spice} & Ridge & \textsc{Lasso} \\
		\hline
		50  & 7.74 (0.90)  &  21.04 (0.90)  &  8.13 (0.90) \\
     	100 & 6.33 (0.90)  &  9.83 (0.90)  &  6.40 (0.90) \\
		200 & 5.48 (0.90)  &  8.02 (0.90)  &  5.56 (0.90)
	\end{tabular}
	\end{center}
        \label{tab:exp_sparse_coverage}
\end{table}

The runtime for the offline learning approach using Ridge or \textsc{Lasso}, is $\mathcal{O}(nKp^2)$ which is similar to the runtime for the online \textsc{Spice} method $\mathcal{O}(nLp^2)$, where $L=3$ in this example. The average runtimes are reported in Table~\ref{tab:exp_sparse_cputimes}. While all three methods scale linearly in $n$, using a cross-validated \textsc{Lasso} predictor is slower in this implementation.

\begin{table}
	\caption{Average runtimes in [s].}
	\begin{center}
	\begin{tabular}{c|c|c|c}
		$n'$ & \textsc{Spice} & Ridge & \textsc{Lasso} \\
		\hline
		50  & 0.85  &  0.93  &  6.01 \\
     	100 & 1.70  &  1.87  &  13.26 \\
		200 & 3.50  &  3.79  &  25.17
	\end{tabular}
	\end{center}
        \label{tab:exp_sparse_cputimes}
\end{table}












\subsection{Global ozone data}

The ozone density determines the transmission of ultraviolet radiation
through the atmosphere which has an important impact on biochemical
processes and health. For this reason, measuring the total column
ozone has been of interest to scientists for decades. In 1978, the
Nimbus-7 polar orbiting satellite was launched, equipped with a total
ozone mapping spectrometer. The satellite was sun
synchronous, and due to the rotation of the Earth, its scan covered
the entire globe in a 24 hour period with a certain spatial
resolution \cite{McPetersEtAl1996_nimbus}. For illustrative purposes,
we consider a set of $n_0 = 173~405$ spatial samples of ozone density $y$ from the
satellite, measured in Dobson units (DU) and recorded on October 1st, 1988,
cf. \cite{Cressie&Johannesson2008_fixed}. The spatial coordinates $\x$
were transformed from longitude and latitude using the area-preserving Mollweide
map projection \cite{Snyder1987_map}.

For $\regc(\x)$, we use the Laplace operator
basis with $m=80$ so that $q = 6400$. The boundaries were set
slightly larger than those given by the Mollweide projection: $L_1=1.15
\cdot 2\sqrt{2}R$ and $L_2=1.15 \cdot\sqrt{2}R$, where $R$ is the radius of the Earth. 

In the first experiment, the
training is performed using $n=n_0$ samples and the ozone density is
predicted on a fine spatial scale. Note that this dataset is nearly
three orders of magnitude larger than that used in the previous example,
which makes it too time consuming to implement the offline cross-validation method due to its computational requirement. Therefore we only evaluate the \textsc{Spice} predictor here.  Fig.~\ref{fig:ozone_gaps} illustrates both the
training samples and the predicted ozone density $\what{y}(\x)$. It can
be seen that the satellite data is not uniformly sampled
and, moreover, it contains significant gaps. The predictions in these
gapped areas appear to interpolate certain nontrivial patterns.
\begin{figure*}
\centering
   \begin{subfigure}[b]{0.49\textwidth}
   \includegraphics[width=1\linewidth]{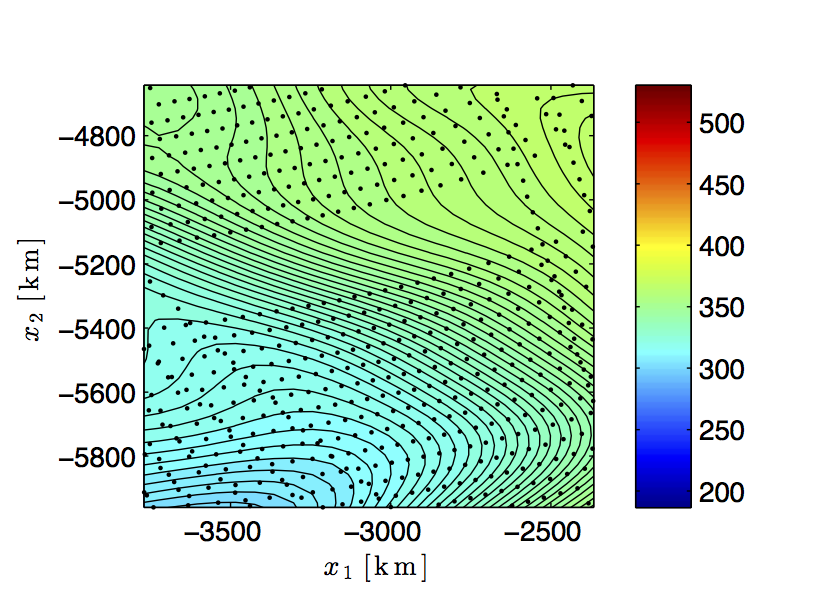}
\end{subfigure}
~
\begin{subfigure}[b]{0.49\textwidth}
   \includegraphics[width=1\linewidth]{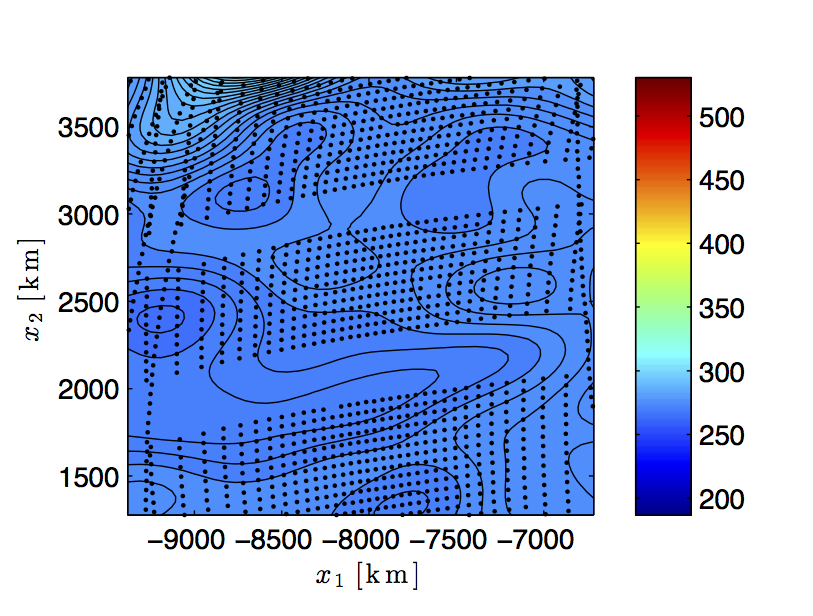}
\end{subfigure}
\caption[TEST]{Predicted ozone density $\hat{y}(\x)$ in [DU] and training
  samples (dots) for two  different areas. Note both the irregular and
  gapped sampling pattern.}
\label{fig:ozone_gaps}
\end{figure*}

In the second experiment, we evaluate the inferential properties of
the \textsc{Spice} predictor by learning from a small random subset of the data. We use $n=2n'=17~340$ samples,
or approximately 
$10\%$ of the data. The resulting predictions exhibit discernible patterns as in
Fig.~\ref{fig:ozone_validation}.
\begin{figure*}
\centering
   \begin{subfigure}[b]{0.49\textwidth}
   \includegraphics[width=1\linewidth]{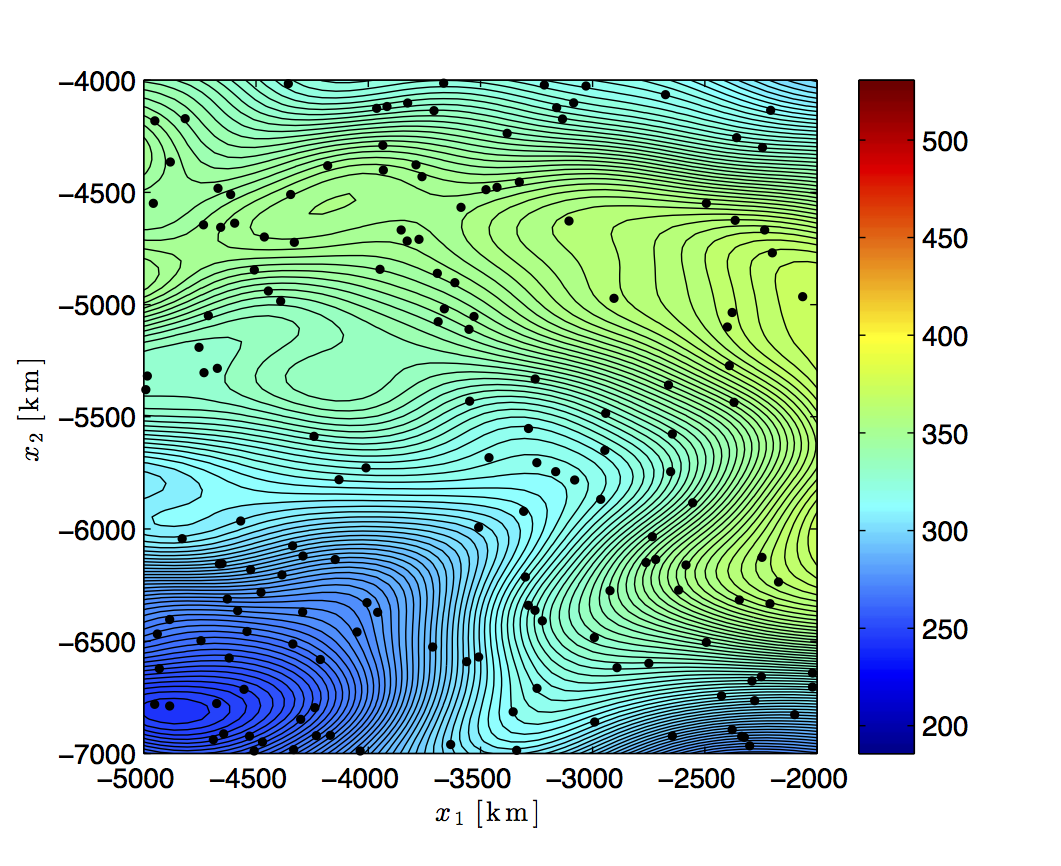}
\end{subfigure}
~
\begin{subfigure}[b]{0.49\textwidth}
   \includegraphics[width=1\linewidth]{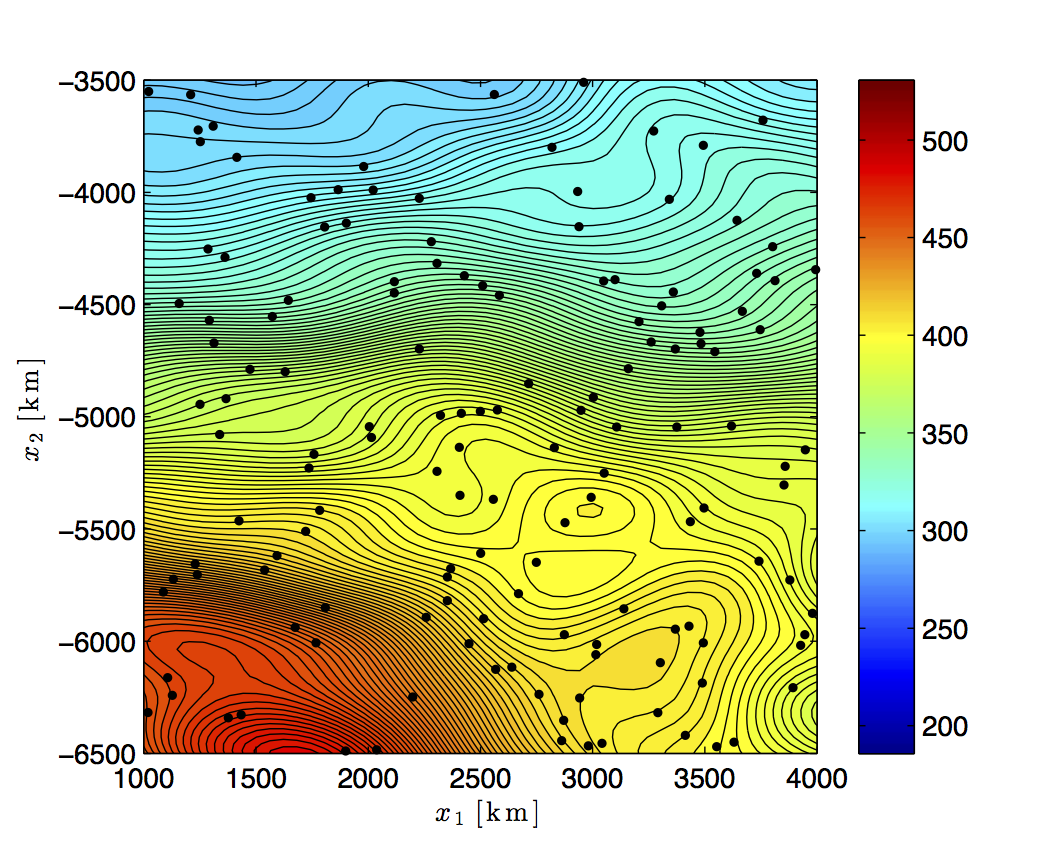}
\end{subfigure}
\caption[TEST]{Predicted ozone density $\what{y}(\x)$ in [DU] and training
  samples (dots) for two  different areas. The predictor was trained
  using a randomly selected dataset consisting of $5\%$ of the original data.}
\label{fig:ozone_validation}
\end{figure*}
For a comparison Fig.~\ref{fig:validate_subfig} zooms in on a region with gapped data, also highlighted in Fig.~\ref{fig:ozone_gaps}. Note that the predictions are consistent with the full data case.
\begin{figure}
  \begin{center}
    \includegraphics[width=1.00\columnwidth]{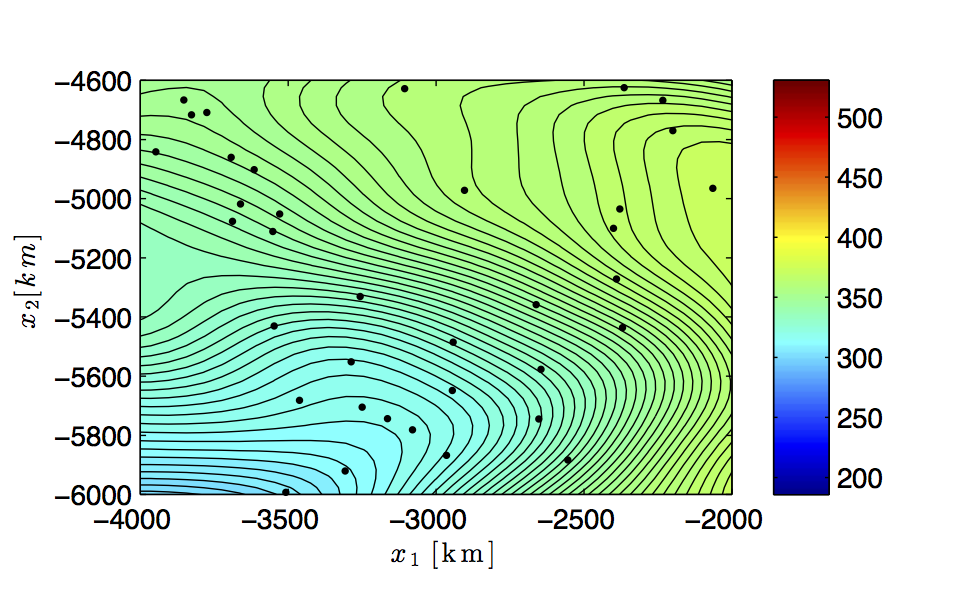}
  \end{center}
  \caption{Predicted ozone density $\hat{y}(\x)$ in [DU] and training
    samples (dots). Here $n \approx 0.05 n_0$ randomly selected training samples are used in contrast with the predictions shown in Fig.~\ref{fig:ozone_gaps}, which uses the full dataset $n=n_0$.}
  \label{fig:validate_subfig}
\end{figure}

The remaining $\bar{n} = n_0 - n =  164~735$ samples are used for validating
the predictor. Its risk is estimated by the out-of-sample prediction errors,
$$\what{\risk} = \Ehat \bigl[|y_i -
  \hat{y}(\x_i) |^2 \bigr],$$
which we translate to Dobson units by taking the square-root. The result was a root-risk of $6.74$ DU. In addition, the confidence interval
$C(\x)$ with $\kappa = 0.90$ had a length of $19.44$ DU and an empirical coverage of 0.90. Both performance metrics compare well with the
dynamic range of the data, which spans [179.40,
542.00] or $362$ DU.  Fig.~\ref{fig:histogram} also shows that the empirical
distribution of the prediction errors is symmetric. These results
illustrate the ability of the proposed method to learn, predict and
infer in real, large-scale datasets.

\begin{figure}
  \begin{center}
    \includegraphics[width=1.00\columnwidth]{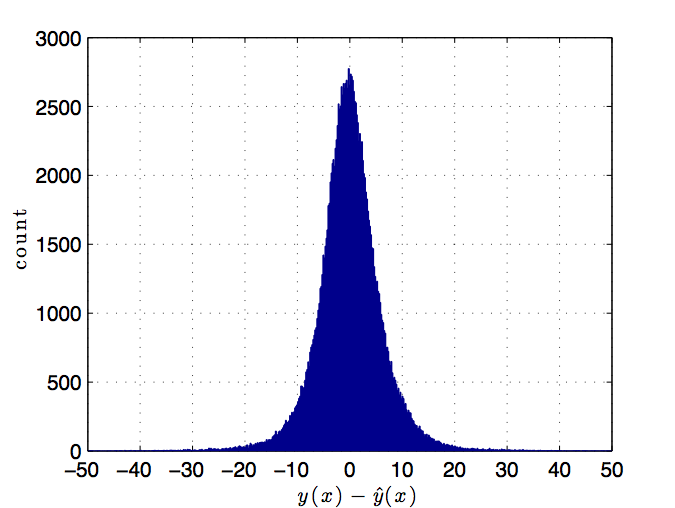}
  \end{center}
  \caption{Histogram of the prediction errors for $\bar{n} = 164\:735$
    samples in [DU]. The predicted samples belong to the interval [179.40,
542.00].}
  \label{fig:histogram}
\end{figure}


\section{Conclusions}

In this paper we considered the problem of online learning for
prediction problems with large and/or streaming data
sets. Starting from a flexible \LC{} predictor, we formulated a convex
covariance-fitting methodology for learning its hyperparameters. The
hyperparameters constrain individual weights of the predictor, similar
to the automatic relevance determination framework.

It was shown that
using the learned hyperparameters results in a predictor with desirable
computational and distribution-free properties. We denote it as the
\textsc{Spice} predictor. It was implemented online with a runtime that
scales
linearly in the number of samples, its memory requirement is
constant, it avoids local minima issues, and prunes away redundant feature dimensions without
relying on assumed properties of the data distributions. In conjunction with the split
conformal approach, it also produces distribution-free prediction confidence
intervals. Finally, the \textsc{Spice} predictor performance was demonstrated on both real
and synthetic datasets, and compared with the offline cross-validation
approach.

In future work, we will investigate input-dependent confidence
intervals, using the locally-weighted split conformal approach.


\small{
\begin{appendices}

\section{Derivation of bounds}

\subsection{\textsc{Lasso} bound \eqref{eq:bound_lasso}}
\label{app:bound_lasso}

We expand the empirical risk of $\what{\coeff}$ by
\begin{equation*}
\begin{split}
\erisk &= \Ehat\left[ | ( y_i - \what{y}_\star(\x_i) )
  - ( \what{y}(\x_i) - \what{y}_\star(\x_i)  ) |^2 \right] \\
&= \erisk_\star + \Delta_\star - 2\Ehat[  \eps_i ( \what{y}(\x_i) -
\what{y}_\star(\x_i) ) ] \\
&= \erisk_\star + \Delta_\star - \frac{2}{n} \epsvec^\T \regM (\what{\coeff}
- \coeff_\star).
\end{split}
\end{equation*}
Using Hölder's inequality along with the triangle inequality yields
\begin{equation*}
\begin{split}
\epsvec^\T \regM (\what{\coeff}
- \coeff_\star) &\leq | \epsvec^\T \regM (\what{\coeff}
- \coeff_\star) | \leq \| \regM^\T \epsvec \|_\infty \| \what{\coeff}
- \coeff_\star\|_1 \\
&\leq \eps_\star ( \| \what{\coeff} \|_1 + \| \coeff_\star \|_1 ).
\end{split}
\end{equation*}
Thus, the prediction divergence is bounded by
\begin{equation}
\Delta_\star \leq \erisk  - \erisk_\star + \frac{2\eps_\star}{n} ( \| \what{\coeff} \|_1 + \| \coeff_\star \|_1 ).
\label{eq:div_equality}
\end{equation}

Next, by inserting $\what{\coeff}$ and $\coeff_\star$ into the cost
function of \eqref{eq:coeffmin_lasso}, we obtain 
\begin{equation}
\erisk  - \erisk_\star \leq \covpsc\left( \| \coeff_\star
  \|_1 - \|\what{\coeff} \|_1 \right) 
\label{eq:risk_inequality}
\end{equation}
Applying this inequality along with \eqref{eq:assumption_lasso} to
\eqref{eq:div_equality}, gives
\begin{equation}
\begin{split}
\Delta_\star &\leq \covpsc\left( \| \coeff_\star
  \|_1 - \|\what{\coeff} \|_1 \right)  + \covpsc ( \|
\what{\coeff} \|_1 + \| \coeff_\star \|_1 ) \\
&= 2\covpsc\| \coeff_\star \|_1.
\end{split}
\end{equation}
Cf. \cite{Buhlmann&VanDeGeer2011_highdim} for the case when the unknown data generating mechanism belongs to a sparse linear model class.

\subsection{\textsc{Spice} bound \eqref{eq:bound_spice}}
\label{app:bound_spice}

For notational simplicity, we write 
 $$g(\coeff) = \| \weight \odot \coeff \|_1 = \sum^p_{j=1} \weightsc_j |w_j|,$$ since $u=0$. By inserting
 $\what{\coeff}$ and $\coeff_\star$ into the cost
function of \eqref{eq:coeffmin_spice}, we have that
\begin{equation}
\begin{split}
\erisk^{1/2}  - \erisk^{1/2}_\star  &\leq \frac{1}{\sqrt{n}}\left( g(\coeff_\star) -
g(\what{\coeff}) \right) \leq  \frac{1}{\sqrt{n}} g(\coeff_\star).
\end{split}
\label{eq:spice_inquality}
\end{equation}
Multipling the inequality by $\erisk^{1/2}  + \erisk^{1/2}_\star $ and rearranging,
yields
\begin{equation}
\begin{split}
\erisk  - \erisk_\star &\leq   \frac{1}{\sqrt{n}}\left(\erisk^{1/2}  +
  \erisk^{1/2}_\star\right) ( g(\coeff_\star) - g(\what{\coeff}) ) \\
&\leq \frac{1}{\sqrt{n}} \underbrace{\left( 2\erisk^{1/2}_\star  + \frac{1}{\sqrt{n}} 
g(\coeff_\star)  \right)}_{\triangleq f(\coeff_\star)}  (
g(\coeff_\star) - g(\what{\coeff}) ) , \\
&= \sum_{j} \frac{1}{\sqrt{n}} f(\coeff_\star) \weightsc_j (|w_{\star,j}| - |\what{w}_j|)
\end{split}
\label{eq:spice_inequalityalt}
\end{equation}
where the second inequality follows from using $\erisk^{1/2}  \leq
\erisk^{1/2}_\star + 
\frac{1}{\sqrt{n}} g(\coeff_\star)$ in \eqref{eq:spice_inquality}. 

Inserting \eqref{eq:spice_inequalityalt} into \eqref{eq:div_equality}, yields
\begin{equation}
\begin{split}
\Delta_\star \; &\leq \;  \sum_{j} \frac{1}{\sqrt{n}} f(\coeff_\star)
\weightsc_j (|w_{\star,j}| - |\what{w}_j|) +
\frac{2\eps_\star}{n}(|w_{\star,j}| + |\what{w}_j|) .
\label{eq:div_equality_spice}
\end{split}
\end{equation}
Given \eqref{eq:assumption_spice}, we have that
\begin{equation*}
\begin{split}
\weightsc_j &\geq \frac{\eps_\star }{ \sqrt{n\erisk_\star}}
\geq\frac{2\eps_\star }{\sqrt{n}(2\erisk^{1/2}_\star + \frac{1}{\sqrt{n}} g(\coeff_\star))} =
\frac{2\eps_\star}{\sqrt{n}f(\coeff_\star)}
\end{split}
\end{equation*}
By re-arranging and dividing by $n$, we obtain the following inequality
\begin{equation*}
\begin{split}
\frac{1}{\sqrt{n}} f(\coeff_\star) \weightsc_j &\geq
\frac{2\eps_\star}{n}, \quad \forall j > u.
\end{split}
\end{equation*}
Therefore \eqref{eq:div_equality_spice} can be bounded by
\begin{equation*}
\begin{split}
\Delta_\star \; &\leq \; \sum_{j} \frac{2}{\sqrt{n}} f(\coeff_\star)
\weightsc_j |w_{\star,j}| \\
&= \frac{2}{\sqrt{n}} \left( 2\erisk^{1/2}_\star  + \frac{1}{\sqrt{n}} 
g(\coeff_\star)  \right) g(\coeff_\star) \\
&=4 \sqrt{\frac{\erisk_\star}{n}} g(\coeff_\star)  + \frac{2}{n} g^2(\coeff_\star)
\end{split}
\end{equation*}

Note that in the special case when the prediction errors of $\what{y}_\star(\x)$ are i.i.d. $\eps_i \sim \mathcal{N}(0, \sigma^2)$, and the
regressors are normalized as $\| \regcol_j \|^2_2 \equiv n$, we have that
\begin{equation}
\sqrt{\sigma^2(2 \ln p +\delta)} \geq \frac{\eps_\star}{\sqrt{n}},
\label{eq:eps_max_gaussian}
\end{equation}
with probability greater than $1 - 2 \exp(-\delta/2)$ where $\delta$ is a positive constant, cf. \cite[ch.~6.2]{Buhlmann&VanDeGeer2011_highdim}. 
In this setting $\erisk_\star$ is a consistent estimate of
$\sigma^2$. Therefore \eqref{eq:assumption_spice} is satisfied with
high probability in this case if the elements of
$\weight$ are multiplied by a factor $c\sqrt{2\ln p
 + \delta}$, so that
$$\weightsc_j = \frac{1}{\sqrt{n}} \| \regcol_j \|_2 \times  c\sqrt{2\ln p
 + \delta}, $$ 
for some $c >1$.

 \section{Optimal weights \eqref{eq:optimallin}}\label{app:optimalfilter}

The conditional risk can be decomposed as
\begin{equation}\label{eq:MSE}
\begin{split}
  \risk(\lin|\x) &= \E[|y - \what{y}(\x)|^2 | \X, \x] \\
&= \Var[y  - \what{y}(\x) | \X, \x ] + (\text{Bias}[\what{y}(\x)])^2 \\
&= \Var[y | \x] + \Var[\what{y}(\x) | \X]  - 2\Cov[y,\what{y}(\x) | \X, \x]] \\
&= K + \Var[\lin^\top \y | \X] - 2\Cov[y,\lin^\top \y |  \X, \x] \\
&= K + \lin^\top \ycov \lin - 2\lin^\top \ycross
\end{split}
\end{equation}
using \eqref{eq:statisticalmoments_cov}. Here $K$ is a constant and
the bias vanishes due to the unbiasedness constraint on $\lin$.  This
constraint, in turn, can be expressed as
$$(\ux^\T(\x)  - \lin^\T \U ) \coeff_0 = 0, \quad \forall \coeff_0,$$
or equivalently
$$\U^\top \lin = \ux(\x).$$
Note that a weak assumption in our case is $\ux(\x) \in \mathcal{R}(\U^\top)$,
especially when $n > u$ and $\ux(\x) \equiv 1$. Hence there exist $\lin$ that satisfy the equality above.

Thus the problem \eqref{eq:BLUP_problem} can be
reformulated using the convex Lagrangian
\begin{equation}\label{eq:BLUP_lagrange}
\begin{split}
L(\lin, \lag) = \lin^\top \ycov \lin - 2\lin^\top \ycross
+ 2(  \U^\top \lin - \ux(\x) )^\top \lag,
\end{split}
\end{equation}
where $\lag$ is the $u \times 1$ vector of Lagrange multipliers.

A stationary point of the Lagrangian in \eqref{eq:BLUP_lagrange} with respect to $\lin$ and
$\lag$ satisfies 
\begin{equation*}
\begin{cases}
\ycov \lin - \ycross + \U \lag &= \mbf{0}, \\
\U^\top \lin - \ux(\x)&= \mbf{0}. 
\end{cases}
\end{equation*}
The first equality can be expressed as $\lin = \ycov^{-1}( \ycross - \U \lag)
$. We
can insert the first equality into the second and solve for the Lagrange multipliers 
\begin{equation*}
\lag = -(\U^\top \ycov^{-1} \U)^-( \ux(\x) - \U^\top \ycov^{-1} \ycross ).
\end{equation*}
Then we have
\begin{equation*}
\begin{split}
\lin &= \ycov^{-1}\ycross + \ycov^{-1}\U (\U^\top \ycov^{-1} \U)^-
\ux(\x) \\
&\quad - \ycov^{-1}\U (\U^\top \ycov^{-1} \U)^- \U^\top \ycov^{-1}\ycross.
\end{split}
\end{equation*}
Noting that the orthogonal projector is given by $\proj^{\perp} =
\mbf{I} - \U (\U^\top \ycov^{-1} \U)^-\U^\top \ycov^{-1} $ concludes
the proof.

 \section{Linear regression form \eqref{eq:coeffmin_smooth}}\label{app:equivalence}
 
Let $\reg(\x)$ be partitioned as in \eqref{eq:regressor_smooth}. Then 
$$\what{y}(\x) = \lin^\T(\x) \y =\reg^\T(\x) \what{\coeff}$$
in combination with \eqref{eq:optimallin} allows us to identify the weights:
\begin{equation}
\what{\coeff} = \begin{bmatrix} \what{\coeff}_0 \\
  \what{\coeff}_1 \end{bmatrix} =  \begin{bmatrix}  (\U^\top \ycov^{-1} \U)^\dagger \U^\T  \ycov^{-1}  \y  \\
  \Lam \regcM^\T \ycov^{-1} (\y - \U \what{\coeff}_0 ) \end{bmatrix}.
\label{eq:app_equivalentform}
\end{equation}

Next, define the problem
\begin{equation}
\min_{\coeff_0, \: \coeff_1, } \;  \covpsc^{-1}_0\| \y -
\U \coeff_0 - \regcM \coeff_1  \|^2_2 +\| \coeff_1 \|^2_{\Lam^{-1}}
\label{eq:covmatch_equiv}
\end{equation}
for which the minimizing $\coeff_1$, namely,
\begin{equation*}
\begin{split}
\what{\coeff}_1 &= \covpsc^{-1}_0( \covpsc^{-1}_0\regcM \regcM^\top + \Lam^{-1} )^{-1} \regcM^\top (\y - \U \nusn) \\
&= \covpsc^{-1}_0(\Lam - \Lam \regcM^\top \ycov^{-1} \regcM \Lam  ) \regcM^{\top}( \y  - \U \nusn  ) \\
&= \covpsc^{-1}_0\Lam \regcM^\top \ycov^{-1}  ( \ycov -  \regcM \Lam \regcM^{\top})( \y  - \U \nusn  ) \\
&= \Lam\regcM^\top \ycov^{-1}  ( \y  - \U \nusn  ),
\end{split}
\end{equation*}
has the same form as in
\eqref{eq:app_equivalentform} with $\coeff_0$ still to be determined. In the above calculation we made use of
the matrix inversion lemma. Inserting $\what{\coeff}_1$ back into
\eqref{eq:covmatch_equiv} yields the concentrated cost function
\begin{equation*}
\begin{split}
& \covpsc^{-1}_0\| \y - \U \coeff_0 - \regcM \what{\coeff}_1   \|^2_2 +\|\what{\coeff}_1 \|^2_{\Lam^{-1}} \\
=&\covpsc^{-1}_0\| (\mbf{I}_N - \regcM \Lam \regcM^\top  \ycov^{-1})( \y - \U  \coeff_0)  \|^2_2 \\
&\quad +\| \Lam\regcM^\top \ycov^{-1}  ( \y  - \U \coeff_0 )
\|^2_{\Lam^{-1}} \\
=& \covpsc_0 ( \y  - \U  \coeff_0 )^\top \ycov^{-1} \ycov^{-1}  ( \y  - \U  \coeff_0 )\\
&\quad +  ( \y  - \U  \coeff_0 )^\top \ycov^{-1} \regcM \Lam  \regcM^\top
\ycov^{-1} ( \y  - \U  \coeff_0 )  \\
=& ( \y  - \U  \coeff_0 )^\top \ycov^{-1} (\y - \U  \coeff_0 ) .
\end{split}
\end{equation*}
Minimizing with respect to $\coeff_0$ yields $\what{\coeff}_0$ in
\eqref{eq:app_equivalentform}. Finally, multiplying the cost
function \eqref{eq:covmatch_equiv} by $\covpsc_0 / n$ yield the desired form.

\section{Convexity of learning problem \eqref{eq:covariancematch_original} } \label{app:convexity}

By expanding the criterion \eqref{eq:covariancematch_original} we obtain
\begin{equation}
\wtilde{\y}^\T \ycov^{-1} \wtilde{\y} + \tr\{ \ycov \}
\label{eq:eq:covariancematch_original_proof}
\end{equation}
where $\wtilde{\y} = \y - \U \coeff_0$.

We note that $\ycov$ in \eqref{eq:statisticalmoments_cov} is a linear function of $\covp$. Thus the normalization constraint \eqref{eq:covariancematch_original_constraint} can be written as a linear equality constraint
\begin{equation}
 n \covpsc_0  + \sum^q_{j=1}  \|\regcol_j\|^2 \covpsc_j = \rho
\label{eq:covariancematch_original_constraint1}
\end{equation}
Next, we define an auxiliary variable $\alpha$ that satisfies
$$\alpha \geq \wtilde{\y}^\T \ycov^{-1} \wtilde{\y},$$
or equivalently 
\begin{equation}
\begin{bmatrix}
\alpha & \wtilde{\y}^\T \\
\wtilde{\y} & \ycov
\end{bmatrix} \succeq \0.
\label{eq:covariancematch_original_constraint2}
\end{equation}

Using the auxiliary variable and the definition of $\ycov$, we can therefore reformulate the learning problem
$$\min_{\alpha, \: \covp} \; \alpha + n \covpsc_0  + \sum^q_{j=1}  \|\regcol_j\|^2 \covpsc_j,$$
which has a linear cost function and is subject to the constraints \eqref{eq:covariancematch_original_constraint1} and \eqref{eq:covariancematch_original_constraint2}. This problem is recognized as a semidefinite program, i.e. it is convex. See \cite{LoboEtAl1998_applications} and \cite{StoicaEtAl2011_spice}.

\section{\textsc{Spice}
  predictor \eqref{eq:spice_predictor}}\label{app:spice_predictor}

The derivation follows in three steps. First we show that dropping the constraint \eqref{eq:covariancematch_original_constraint} yields 
the simpler unconstrained problem in
\eqref{eq:covariancematch_original}. Next, we prove that the fitted hyperparameters in the simplified problem yields the same predictor. Finally, we show how the corresponding \LR{} form arises as a consequence.

By dropping the constraint
\eqref{eq:covariancematch_original_constraint} we may consider the
problem in the following form:
\begin{equation}
\what{\covp} = \argmin_{\covp} \: \tr\{\scov \ycov^{-1}\} +
\tr\{\ycov\},
\label{eq:covariancematch}
\end{equation} 
where $\scov$ is the unnormalized sample covariance matrix.

We now prove that, $\covp^\star \propto \what{\covp}$. Begin by defining a constant $\kappa > 0$, such that $\text{tr}\{\scov \ycov^{-1}(\what{\covp})\} = \kappa^2 \text{tr}\{ \ycov(\what{\covp})\}$ at the minimum of \eqref{eq:covariancematch}. We show that $\kappa=1$ is the only possible value and so both terms in \eqref{eq:covariancematch} equal each other at the minimum.

Let $\tilde{\covp} = \kappa \what{\covp}$, and observe that the cost \eqref{eq:covariancematch} is then bounded by
\begin{equation*}
\begin{split}
(\kappa^2+1) \text{tr}\{ \ycov(\what{\covp})\} &\leq \text{tr}\{\scov \ycov^{-1}(\tilde{\covp} )\} + \text{tr}\{\ycov(\tilde{\covp} )\} \\
&= \kappa^{-1}\text{tr}\{\scov \ycov^{-1}(\what{\covp} )\} +\kappa\text{tr}\{ \ycov(\what{\covp})\}\\
&= 2\kappa \text{tr}\{ \ycov(\what{\covp})\}.
\end{split}
\end{equation*}
Thus $\kappa$ must satisfy $\kappa^2+1 \leq 2 \kappa$, or $(\kappa-1)^2 \leq 0$. Therefore $\kappa= 1$ is the only solution and both terms must be equal at the minimum. We can thus re-write the minimization of \eqref{eq:covariancematch} as the following problem
\begin{equation}\label{eq:covariancematch_alt1}
\begin{split}
\min & \quad \alpha \\
\text{subject to} & \quad \text{tr}\{\scov \ycov^{-1}\} = \alpha, \; \text{tr}\{ \ycov \} = \alpha,
\end{split}
\end{equation}
with minimizer $\hat{\covp}$ and where $\alpha > 0$ is an auxiliary variable.

Next, consider an equivalent problem to \eqref{eq:covariancematch_alt1} obtained by re-defining the variables as $\tilde{\covp} =  \rho \alpha^{-1}\covp$. Then $\text{tr}\{\scov \ycov^{-1}(\covp)\} = \rho \alpha^{-1} \text{tr}\{\scov \ycov^{-1}(\tilde{\covp})\}$ and $\text{tr}\{ \ycov(\covp) \} = \alpha \rho^{-1} \text{tr}\{ \ycov(\tilde{\covp}) \}$, so that the equivalent problem becomes
\begin{equation}\label{eq:covariancematch_alt2}
\begin{split}
\min & \quad \beta \\
\text{subject to} & \quad \text{tr}\{\scov \ycov^{-1}\} = \beta, \; \text{tr}\{ \ycov \} = \rho,
\end{split}
\end{equation}
where $\beta = \alpha^2 \rho^{-1}$. The minimizer of the equivalent
problem \eqref{eq:covariancematch_alt2} is therefore $\tilde{\covp}
\propto \what{\covp}$. Problem \eqref{eq:covariancematch_alt2} is
however identical to the constrained problem 
\begin{equation*}
\begin{split}
\min & \quad \text{tr}\{\scov \ycov^{-1}\}  \\
\text{subject to} & \quad \text{tr}\{ \ycov \} = \rho,
\end{split}
\end{equation*}
 whose minimizer is $\tilde{\covp} = \covp^\star$, which follows from expanding the cost in \eqref{eq:covariancematch_original} and the constraint \eqref{eq:covariancematch_original_constraint}.

Thus we proved that $\covp^\star \propto \what{\covp}$. Next, note that the
optimal \LC{} predictor based on \eqref{eq:optimallin} is invariant to uniform scaling of
$\covp$. That is, $\what{y}(\x; \covp) = \what{y}(\x;
c\covp)$ for all $c > 0$.  The result follows readily by inspection of the minimizer \eqref{eq:coeffmin_smooth}, given in \eqref{eq:app_equivalentform}. Therefore $$\what{y}(\x; \covp^\star) = \what{y}(\x ; \what{\covp}).$$

Finally, consider the following augmented problem
\begin{equation}
\min_{\coeff_0, \: \coeff_1, \: \covp} \;  \covpsc^{-1}_0\| \y -
\U \coeff_0 - \regcM \coeff_1  \|^2_2 +\| \coeff_1 \|^2_{\Lam^{-1}} + \text{tr}\{
\ycov \}.
\label{eq:covmatch_aug}
\end{equation}
Solving for $\coeff_0$ and $\coeff_1$ yields the minimizer
\eqref{eq:coeffmin_smooth}, i.e.,
\eqref{eq:app_equivalentform}. Moreover, by inserting the minimizing $\coeff_1$ back into
\eqref{eq:covmatch_aug} we obtain the concentrated cost function
\begin{equation*}
\begin{split}
&\covpsc^{-1}_0\| (\mbf{I}_n - \regcM \Lam \regcM^\top  \ycov^{-1})( \y - \U \nusn)  \|^2_2 \\
&\quad +\| \Lam\regcM^\top \ycov^{-1}  ( \y  - \U \nusn ) \|^2_{\Lam^{-1}} + \text{tr}\{ \ycov \}\\
=& \covpsc_0 ( \y  - \U \nusn )^\top \ycov^{-1} \ycov^{-1}  ( \y  - \U \nusn )\\
&\quad +  ( \y  - \U \nusn )^\top \ycov^{-1} \regcM \Lam  \regcM^\top
\ycov^{-1} ( \y  - \U \nusn ) + \text{tr}\{ \ycov \} \\
=& \text{tr}\{( \y - \U \nusn ) ( \y  - \U \nusn )^\top \ycov^{-1} \} + \text{tr}\{
\ycov \}
\end{split}
\end{equation*}
which is equal to that in \eqref{eq:covariancematch}. Thus the
augmented problem \eqref{eq:covmatch_aug} enables us to obtain both
$\what{\covp}$ and $\what{\coeff}$. 

Using the result above, we may alternatively solve for $\covp$ first. The second and third terms in \eqref{eq:covariancematch} can be written as
\begin{equation*}
\| \coeff_1 \|^2_{\Lam^{-1}} =  \sum^q_{k=1}
\frac{1}{\covpsc_k}w^2_{u+k}
\end{equation*} 
and
\begin{equation*}
\begin{split} 
\text{tr}\{ \ycov \} &= \text{tr}\{ \regcM^\top \regcM \Lam \}  + \tr\{ \covpsc_0 \I_n \} \\
&= \sum^q_{k=1} \| \regcol_{u+k} \|^2_2 \covpsc_k + n \covpsc_0,
\end{split}
\end{equation*}
respectively. Then the minimizing hyperparameters $\covp$ in \eqref{eq:covmatch_aug} can be expressed in closed-form:
\begin{equation*}
\hat{\covpsc}_k = \begin{cases} \|   \y - \regM \z\|_2/\sqrt{n}, \quad k
  = 0.\\
|w_{u+k}|/\| \regcol_{u+k} \|_2, \quad k = 1, \dots, q.\end{cases} 
\end{equation*}
Inserting the expression back in to \eqref{eq:covmatch_aug} 
yields a concentrated cost function
$$\sqrt{\| \y - \regM \coeff \|^2_2} + \sum^p_{j=u+1} \frac{1}{\sqrt{n}}\| \regcol_j \|_2 |w_j|$$
which, after dividing by $n^{-1/2}$, equals that in
\eqref{eq:coeffmin_spice}. Thus the right hand side of
\eqref{eq:spice_predictor} yields $\what{y}(\x ; \what{\covp})$ when using $\what{\covp}$
from \eqref{eq:covariancematch}.

\section{Online algorithm}
\label{app:spice_alg}

\begin{proof}
We reformulate the convex problem in \eqref{eq:coeffmin_spice} at a
given sample size $n$, using \eqref{eq:variables_spice}. We solve the
problem in
\eqref{eq:coeffmin_spice} via cyclic minimization
\cite{Zachariah&Stoica2015_onlinespice}. That is, we minimize the cost
function, with respect to one variable $w_j$ at a time, while holding
the remaining variables $\{ w_k \}_{k\neq j}$ are held constant \cite{Stoica&Selen2004_cyclic}. 

Recall that $\regcol_j$ is the
$j$th column of $\regM$ and $$\wbar{\y}_j = \mbf{y} -
\sum_{k\neq j} \regcol_k \check{w}_k,$$ where $\check{w}_k$ is the
current estimate of $w_k$. (When $n = 0$, the initial estimate
$\check{w}_j$ is set to 0.) Starting from the cost in
\eqref{eq:coeffmin_spice}, we define an equivalent cost function with
respect to $w_j$:
\begin{equation}
V(w_j) \triangleq \| \wbar{\y}_j  - \regcol_j w_j \|_2 + \weightsc_{j}
|w_j|,
\label{eq:subproblem}
\end{equation}
where 
\begin{equation*}
\weightsc_j =
\begin{cases}
0, \quad  \quad j \leq u\\
 \| \regcol_j \|_2 / \sqrt{n}, \quad  j > u\\
\end{cases}.
\end{equation*}
Now consider two cases:

\emph{Case 1)} when $j=1,\dots, u$: Then $\weightsc_j = 0$ and \eqref{eq:subproblem} can be written as
\begin{equation}\label{eq:Vi_p}
V(w_j) = ( \| \wbar{\y}_j \|^2 + \|\regcol_j\|^2 w^2_j - 2\regcol^\top_j \wbar{\y}_j  w_j  )^{1/2}.
\end{equation}
The minimizer of \eqref{eq:Vi_p} is readily found as $\hat{w}_j = \regcol^\top_j\wbar{\y}_j/\|\regcol_j\|^2$ and 
we get a nonnegative expression $ (\| \regcol_j \|^2 \| \wbar{\y}_j
\|^2 - ( \regcol^\top_j \wbar{\y}_j)^2 )/\| \regcol_j \|^2 \geq 0$ inside the brackets of
$V(\hat{w}_j)$ 
using the Cauchy-Schwartz inequality. Noting that $\wbar{\y}_j = \y - \regM \check{\coeff} +  \regcol_j
\check{w}_j$, we can express the minimizer as
\begin{equation}\label{eq:zi_hat_p}
\begin{split}
\hat{w}_j &= \frac{\regcol^\top_j( \y - \regM \check{\coeff} +  \regcol_j
  \check{w}_j)}{  \|\regcol_j \|^2 } = \frac{\zeta_j + \Gamma_{jj} \check{w}_j}{\Gamma_{jj}},
\end{split}
\end{equation}
where we have defined the vector $$\mbs{\zeta} \triangleq \mbs{\rho} - \mbs{\Gamma} \check{\mbs{\z}}.$$

\emph{Case 2)} when $j=u+1,\dots, p$: We parameterize the
variable $w_j$ as $w_j = s_i r_j$, where $r_j \geq 0$ and $s_j \in \{ -1,1 \}$. Then \eqref{eq:subproblem} becomes
\begin{equation*}
\begin{split}
V(r_j,s_j) &= ( \| \bar{\y}_j  \|^2 + \|\regcol_j\|^2 r^2_j  - 2
\regcol^\top_j \wbar{\y}_j s_j r_j )^{1/2}  + \weightsc_j r_j,
\end{split}
\end{equation*}
The minimizing $s_j$ is given by 
\begin{equation}\label{eq:scalars_phi}
\begin{split}
\hat{s}_j &= \text{sign}(\regcol^\top_j \wbar{\y}_j )\\
&=  \text{sign}(\regcol^\top_j (  \y - \regM \check{\coeff} +  \regcol_j
\check{w}_j ) ) \\
&= \text{sign}(\zeta_j + \Gamma_{jj} \check{w}_j) .
\end{split}
\end{equation}
To write the minimizing $r_i$ in a compact manner, we introduce the
following variables
\begin{equation*}
\begin{split}
\xi &\triangleq \|  \y - \regM \check{\coeff} \|^2 \\
&=\kappa + \check{\coeff}^\top \mbs{\Gamma} \check{\coeff} - 2 \check{\coeff}^\top
\mbs{\rho} 
\end{split}
\end{equation*}
and
\begin{equation}\label{eq:scalars_abc}
\begin{split}
\alpha_j &\triangleq  \|  \wbar{\y}_j  \|^2  \\
&= \|  \y - \regM \check{\coeff} +  \regcol_j
\check{w}_j \|^2  \\
&=\xi + \Gamma_{jj}\check{z}^2_j + 2\check{w}_j \zeta_j, \\
\beta_j  &\triangleq \| \regcol_j \|^2 \\
&=\Gamma_{jj},\\
\gamma_j &\triangleq |\regcol^\top_j  \wbar{\y}_j | \\
&= |\regcol^\top_j  ( \y - \regM \check{\coeff} + \regcol_j \check{w}_j) | \\
&= |\zeta_j + \Gamma_{jj} \check{w}_j |. \\
\end{split}
\end{equation}
The use of this notation enables us to express the concentrated cost function as
\begin{equation}\label{eq:Vi_q}
\begin{split}
V(r_j,\hat{s}_j) &= ( \alpha_j + \beta_j r^2_j - 2
\gamma_j r_j )^{1/2} + \weightsc_j r_j.
\end{split}
\end{equation}
It was shown in \cite[sec.~III]{Zachariah&Stoica2015_onlinespice} that the minimizer of \eqref{eq:Vi_q} is
\begin{equation}
\begin{split}
\hat{r}_j &= \frac{\gamma_j}{\beta_j} - \frac{1}{\beta_j} \left( \frac{ \alpha_j \beta_j - \gamma^2_j}{n-1} \right)^{1/2} ,
\end{split}
\label{eq:scalars_r}
\end{equation}
when $\sqrt{n-1} \gamma_j > \sqrt{ \alpha_j \beta_j - \gamma^2_j
}$. Otherwise $\hat{r}_j  = 0$. More compactly,
\begin{equation}
\hat{w}_j = 
\begin{cases}
\hat{s}_j\hat{r}_j , & \text{if }\sqrt{n-1} \gamma_j >
\sqrt{ \alpha_j \beta_j - \gamma^2_j }\\
0 , & \text{else},
\end{cases}
\label{eq:zi_hat_pq}
\end{equation}
using \eqref{eq:scalars_phi},
\eqref{eq:scalars_abc} and \eqref{eq:scalars_r}. 

In summary, at sample $n$ the cyclic minimizer for problem
\eqref{eq:coeffmin_spice} consists of the iterative application of 
\eqref{eq:zi_hat_p} and \eqref{eq:zi_hat_pq}. As each element $\hat{w}_i$ is updated, the convex cost function in \eqref{eq:coeffmin_spice} decreases
monotonically by a general property of cyclic minimizers. Therefore by
repeating the updates $L$ times, the solution will converge to $\what{\coeff}$ in
\eqref{eq:coeffmin_spice} as $L$ increases. Note that all the variables in \eqref{eq:zi_hat_p} and \eqref{eq:zi_hat_pq}, are based on the variables \eqref{eq:variables_spice}.

\end{proof}

\end{appendices}

}

\small{
\bibliographystyle{ieeetr}
\bibliography{refs_covlearning}
}
\end{document}